\documentclass[preprint]{elsarticle}

\usepackage[american]{babel}
\usepackage[T1]{fontenc}
\usepackage{natbib}
\usepackage{xcolor,graphicx,url}
\usepackage[utf8]{inputenc}
\usepackage{amsmath,amssymb,amsthm}
\usepackage{mathtools}
\usepackage{paralist}

\usepackage{listings}
\usepackage{color}
\usepackage{soulutf8}

\definecolor{newcolor}{rgb}{.8,.349,.1}

\usepackage[nomessages]{fp}
\usepackage{catchfilebetweentags}
\usepackage{pifont}
\usepackage{fixfoot}
\usepackage{array}
\usepackage{booktabs} 
\usepackage{longtable}
\usepackage{pdflscape}
\usepackage{subfig}
\usepackage[font={footnotesize}]{caption}
\usepackage{enumerate}
\usepackage{hyperref}
\usepackage{fixltx2e}
\usepackage[export]{adjustbox}
\graphicspath{ {images/} }

\usepackage{unitsdef}

\newcommand\norm[1]{\left\lVert#1\right\rVert}
\newcommand\abs[1]{\left\lvert#1\right\rvert}

\newcommand{\V}{\mathbf{V}}

\newcommand{\vlambda}{\mathbf{\lambda}}
\renewcommand{\v}{\mathbf{v}}
\newcommand{\w}{\mathbf{w}}
\newcommand{\x}{\mathbf{x}}
\newcommand{\X}{\mathbf{X}}

\newcommand{\zero}{\mathbf{0}}

\DeclareMathOperator{\sgn}{sgn}

\DeclareMathOperator*{\argmin}{arg\,min}

\newtheorem{theorem}{Theorem}
\newtheorem{hypothesis}{Hypothesis}
\newtheorem{empirical_evidence}{Empirical evidence}
\newtheorem{technique}{Technique}

\usepackage{xspace}
\newcommand{\xQuant}{xQuant\xspace}

\usepackage{hyphenat}
\journal{Journal of Neurocomputing}

\begin{document}

  \begin{frontmatter}

    \title{Computational Cost Reduction in Learned Transform Classifications}

    \author[mec]{Emerson Lopes Machado\corref{cor1}} \ead{bmi.machado@gmail.com}
    \cortext[cor1]{Corresponding author:}
    \author[fga]{Cristiano Jacques Miosso} 
    \author[utp]{Ricardo von Borries} 
    \author[est]{Murilo Coutinho} 
    \author[cic]{Pedro de Azevedo Berger} 
    \author[mec]{Thiago Marques} 
    \author[cic,mec]{Ricardo Pezzuol Jacobi} 
    
    \address[mec]{Graduate Program in Mechatronic Systems, University of Brasilia, Brazil}
    \address[fga]{University of Brasilia at Gama, Brazil}
    \address[utp]{Dept. of Electrical and Computer Engineering, University of Texas at El Paso, USA}
    \address[cic]{Dept. of Computer Science, University of Brasilia, Brazil}
    \address[est]{Dept. of Statistics, University of Brasilia, Brazil}

    \begin{abstract}
      We present a theoretical analysis and empirical evaluations of a novel set of techniques for computational cost reduction of classifiers that are based on learned transform and soft-threshold. By modifying optimization procedures for dictionary and classifier training, as well as the resulting dictionary entries, our techniques allow to reduce the bit precision and to replace each floating-point multiplication by a single integer bit shift. We also show how the optimization algorithms in some dictionary training methods can be modified to penalize higher-energy dictionaries. We applied our techniques with the classifier Learning Algorithm for Soft-Thresholding, testing on the datasets used in its original paper. Our results indicate it is feasible to use solely sums and bit shifts of integers to classify at test time with a limited reduction of the classification accuracy. These low power operations are a valuable trade off in FPGA implementations as they increase the classification throughput while decrease both energy consumption and manufacturing cost.
    \end{abstract}


  \end{frontmatter}



\section{Introduction}
  \label{Introduction}
  In image classification, feature extraction is an important step, specially in domains where the training set has a large dimensional space that requires a higher processing and memory resource. A recent trend in feature extraction for image classification is the construction of sparse features, where these features consist in the representation of the signal in an overcomplete dictionary. When the dictionary is learned specific to the input dataset, the classification of sparse features can achieve results comparable to state-of-the-art classification algorithms \citep{Mairal:2012kq}. However, this approach has a drawback at test time, as the sparse coding of the input test sample is computationally intense, being impracticable to embedded applications that have scarce computational and power resources. 

  A recent approach to this drawback is to learn a sparsifying transform from the target image dataset \citep{Fawzi:2014gl,Shekhar:2014ik,Ravishankar:2013jj}. Therefore, the learned classifier has an architecture that can be seen as a feedforward neural network (FFNN) with one hidden layer and no bias. At test time, this approach reduces the sparse coding of the input image to a simple matrix-vector multiplication followed by a soft-threshold, which can be efficiently realized in hardware due to its inherent parallel nature. Nevertheless, these matrix-vector multiplications require floating-point operations, which may have a high cost in hardware, specially in FPGA, as it increases the fabrication cost and demands a higher energy to operate.

  Exploring some properties we derive from these classifiers, we propose a set of techniques to reduce their computational cost at test time, which we divide into four main groups:
  \begin{inparaenum}[(i)] 
    \item decrease the dynamic range of the dictionary first by penalizing the $\ell_2$ norm of its entries at the training phase, then by zeroing out its entries that have absolute values smaller than a trained threshold;
    \item use test images in integer --- which is the same format they are sampled by analog-to-digital converters (ADC) --- instead of their scaled normalized version (floating-point) and thus replace the costly floating-point operations by integer operations, which are cheaper to implement in hardware and do not affect the classification accuracy;
    \item quantize the integer valued test images and thus decrease the number of bits needed to represent them;
    \item and quantize both transform dictionary and classifier by approximating its entries to their nearest power of 2 and thus replace each multiplication by a simple bit shift.
  \end{inparaenum} 
  
  From now on, we refer to this set of techniques as \xQuant. As a study case for \xQuant, we use a recent classification algorithm named Learning Algorithm for Soft-Thresholding classifier (LAST), which learns both the sparse representation of the signals and the hyperplane vector classifier at the same time. Our tests use the same datasets used in the paper that introduces LAST and our results indicate that our techniques reduce the computational cost while not substantially degrading the classification accuracy. Moreover, in a particular dataset we tested, our techniques substantially increased the classification accuracy.

  To the best of our knowledge, this paper presents the first generic approach to reduce the computational cost at test time of classifiers that are based on learned transform. This has a valuable application in embedded systems where power consumption is critical and computational power is restricted. Furthermore, \xQuant dismiss the necessity of using DSPs for intense matrix-vector operations in FPGAs architectures for image classification, lowering the overall manufacturing cost of embedded systems. 

  Even though all simulations we ran to test our techniques were performed on image classification using LAST, our proposed techniques are sufficiently general to be applied on different problems and different classification algorithms that use matrix-vector multiplications to extract features, such as Extreme Learning Machine (ELM) \citep{Huang:2006ut} and Deep Neural Networks (DNN) \citep{Schmidhuber:2015cz}.

\section{Related Work}
  \label{Related Work}
  The literature on reducing the computational cost of classifiers is vast and thus we only present some of the significant trends. Also, it is worth noting that quantization strategies to reduce resource usage of FFNN classifiers implemented in FPGA are not new and have been used in the past century with success. In~\cite{Marchesi:1993fz} for example, a quantization scheme is proposed to eliminate all multiplications during the test time. After training the parameters of a feedforward neural network, they approximate these parameters to a power of two and retrain the network letting only the bias values to change freely in the real domain, as these bias do not participate in multiplications. This reduces each multiplication to a single operation of bit shift. The problem with this approach is that it still relies on floating-point operations, which are costly in applications with limited energy and/or small computational power.

  In~\cite{Courbariaux:2014vk}, \cite{Gupta:2015vp}, and~\cite{Lin:2015ut}, different quantization strategies are presented to allow the use of fixed-point values during the training and test time. These works lack the power reducing benefits from quantization schemes that approaches the network parameters to powers of two as in~\cite{Marchesi:1993fz}\cite{Machado:2015vr}. This was probably an unknown feature to the authors. In~\cite{Lin:2016tx}, the authors start to experiment with quantization schemes that allow a higher computational cost reduction. They quantize the network parameters to have only -1s and 1s to reduce multiplications to simple sign changes with only a small decrease of the classification accuracy. \cite{Courbariaux:2015vy} and~\cite{Courbariaux:2016tm} also follow the same lead. This quantization scheme is drastic and eliminates all multiplications and bit shifts at test time, but may substantially reduce the learning capacity of the neural network. In~\cite{Rastegari:2016tn}, the authors propose a post-processing scheme to approximate both the trained parameters of a CNN and the input images to -1s and 1s. This approach allows the convolutions to be estimated by XNOR and bit-counting operations. Nevertheless, this oversimplification comes with the price of a higher degradation of the classification accuracy compared to the original classifier.

  Our approach differs from these aforementioned in many points. First, it can be easily adapted to any learning algorithm as it does not rely on a specific one, and, thus, can be used in different network architectures and different amounts of neurons. Also, \xQuant can also be applied after training the classifier. Second, it drops all floating-point operations in favor of integer ones. This avoids the costly normalization and denormalization techniques required in floating-point operations. Third, it has an optional strategy to reduce the dynamic range of the parameters during training and consequently reduce the number of bits necessary to store them. This strategy penalizes parameter values that causes an increase in the dynamic range by forcing them to be closer to their average. Fourth, \xQuant does not hurt much the classification accuracy as the approximation to -1s and 1s performed in some of the previously mentioned works.

\section{Overview of Sparse Representation Classification}
  \label{Overview of Sparse Representation Classification}
  In this section, we briefly review both synthetical and analytical sparse representation of signals along with the threshold operation used as a sparse coding approach (Section~\ref{Sparse Representation of Signals}). We also review LAST (Section~\ref{Learning Algorithm for Soft-Thresholding Classifier (LAST)}).

  \subsection{Sparse Representation of Signals}
    \label{Sparse Representation of Signals}
    Let $\mathbf{x} \in \mathbb{R}^{n}$ be a signal vector and $\mathbf{D} \in \mathbb{R}^{n \times N}$ be an overcomplete dictionary. The sparse representation problem corresponds to finding the coefficient vector $\mathbf{z}^* \in \mathbb{R}^{N}$ that minimizes the ${\ell_0}$ norm
    \begin{equation}
      \label{eq:sparse_representation_ell_0}
      \mathbf{z}^{*} = \argmin_{\mathbf{z}} \norm{\mathbf{z}}_{0} \text{ s.t. } \mathbf{x} = \mathbf{D}\mathbf{z},
    \end{equation}
    where $\norm{\cdot}_{0}$ measures the number of nonzero coefficients. Therefore, the signal $\mathbf{x}$ can be synthesized as a linear combination of $k$ nonzero columns of the dictionary $\mathbf{D}$, also called synthesis operator. The solution of (\ref{eq:sparse_representation_ell_0}) requires testing all possible sparse vectors $z$, which is a combination of $N$ entries taken $k$ at a time. This is an NP-hard problem, but an approximate solution can be obtained by using the $\ell_1$ norm instead of the $\ell_0$ norm, i.e.
    \begin{equation}
      \label{eq:sparse_representation_ell_1}
      \mathbf{z}^* = \argmin_{\mathbf{z}} \norm{\mathbf{z}}_{1} \text{ s.t. } \mathbf{x} = \mathbf{D}\mathbf{z},
    \end{equation}
    where $\norm{\cdot}_{1}$ is the $\ell_1$ norm. The solution of (\ref{eq:sparse_representation_ell_1}) can be computed by solving the problem of minimizing the $\ell_1$ norm of the coefficients among all decompositions, which is convex and can be solved efficiently. If the solution of (\ref{eq:sparse_representation_ell_1}) is sufficiently sparse, it will be equal to the solution of (\ref{eq:sparse_representation_ell_0}) \citep{Donoho:2001wc}.

    Sparse coding transform \citep{Ravishankar:2013jj} is another way of sparsifying a signal, where the dictionary is a linear transform that maps the signal to a sparse representation. For example, signals formed by the superposition of sinusoids have a dense representation in the time domain and a sparse representation in the frequency domain. For this type of signal, the Fourier transform is the sparse coding transform. Quite simply, $\mathbf{D^\top} \mathbf{x} = \mathbf{z}$ is the sparse transform of $\mathbf{x}$, where $\mathbf{z}$ is the sparse coefficient vector. In general, the transform $\mathbf{D}$, can be a well structured fixed base such as a DFT or learned specifically to the target problem represented in the training dataset. A learned dictionary can be an overcomplete dictionary learned from the signal dataset, as in \citep{Shekhar:2014ik}, a square invertible dictionary, as in \citep{Ravishankar:2013jj}, or even a dictionary without restrictions on the number of atoms, as in LAST \citep{Fawzi:2014gl}.

    When a signal is corrupted by additive white Gaussian noise (AWGN), its transform will result into a coefficient vector that is not sparse. A common way of making it sparse is to apply a threshold operation to its entries right after the transform, where the entries lower than the specified threshold are set to zero. Among the existing threshold operators, soft-threshold is the one that, in addition to the threshold operation, subtracts the remaining values by the threshold, shrinking them toward zero \citep{Donoho:1994ux}.

    Let $\mathbf{z} = (z_i)^{N}_{i=1}$ be the coefficients of a sparse representation of a signal corrupted by AWGN given by
    \begin{equation}
      \label{noisy_sparse}
      z_i = s_i + \epsilon \, e_i  \;\;\;\;  i = 1,...,N
    \end{equation}
    where $e_i$ is the noise i.i.d.\ as $\mathcal{N}(0,1)$, $\epsilon > 0$ is the noise level, and $s_i$ are the coefficients of the sparse representation of the pure signal.

    Because the $s_i$ coefficients in (\ref{noisy_sparse}) are sparse, there exists a threshold $\alpha$ that can separate most of the pure signal $s_i$ from the noise $e_i$ using the soft-thresholding operator~\citep{Donoho:1994ux}
    \begin{equation}
      \label{eq:soft_threshold}
      h_\alpha(\mathbf{z}) = \sgn(\mathbf{z}) \max(0,\abs{\mathbf{z}} - \alpha),
    \end{equation}
    where $\sgn(\cdot)$ is the sign function. For classification tasks, the best estimate of $\alpha$ can be computed using the training set. 

  \subsection{Learning Algorithm for Soft-Thresholding Classifier (LAST)}
    \label{Learning Algorithm for Soft-Thresholding Classifier (LAST)}
    LAST \citep{Fawzi:2014gl} is an algorithm based on a learned transform followed by a soft-threshold, as described in Section~\ref{Sparse Representation of Signals}. Differently from the original soft-threshold map (\ref{eq:soft_threshold}), LAST uses a soft-threshold version that also sets to zero all negative values, i.e., $h_{\alpha}(\mathbf{z}) = \max(0,\mathbf{z}-\alpha)$, where $\alpha$ is the threshold, also called sparsity parameter. When $\alpha = 0$, this threshold operator can be seen as the \emph{relu} activation function, which has produced good results in deep neural network architectures~\citep{Glorot:2011tm,Nair:2010vq,Maas:2013tn,Zeiler:2013ts}. We chose LAST to be our study case because of its simplicity in the learning process of the sparsifying dictionary and the classifier hyperplane.

    For the training cases $\mathbf{X} = [\mathbf{x}_1| \dots |\mathbf{x}_m] \in \mathbf{R}^{n \times m}$ with labels $\mathbf{y} = [\mathbf{y}_1| \dots |\mathbf{y}_m] \in \{-1,1\}^m$,  the sparsifying dictionary $\mathbf{D} \in \mathbb{R}^{n \times N}$ that contains $N$ atoms and the classifier hyperplane $\mathbf{w} \in \mathbb{R}^N$ are estimated using the supervised optimization
    \begin{equation}
      \label{eq:objective_function_LAST}
      \min_{\mathbf{D},\mathbf{w}} \sum^{m}_{i=1}H(y_i \mathbf{w^\top} h_{\alpha}(\mathbf{D^\top x_i})) + \frac{v}{2}\norm{\mathbf{w}}^2_2,
    \end{equation}
    where $H$ is the hinge loss function $H(x) = \max(0,1-x)$ and $v$ is the regularization parameter that prevents the overfitting of the classifier $\mathbf{w}$ to the training set. At test time, the classification of each test case $\mathbf{x}$ is performed by first extracting the sparse features from the signal $\mathbf{x}$, using 
    \begin{equation}
      \label{eq:LAST_classification_test_time}
      \mathbf{f} = \max(0, \mathbf{D^\top x} - \alpha),
    \end{equation}
    and then by the classification of these features using $c = \mathbf{w^\top f} > 0$, where $c$ is the class returned by the classifier. We direct the reader to \citep{Fawzi:2014gl}, for a deeper understanding of LAST.

\section{Proposed Techniques}
  \label{section: Proposed Techniques}
  In this section we introduce a set of techniques for simplifying the test-time computations of classifiers based on learned transforms and soft-threshold. We start by describing in Section~\ref{sec:datasets} the dataset of images to which we apply the proposed techniques for validation. Next, we present in Section~\ref{Theoretical Results on Computational Cost Reduction} our main theoretical findings supporting \xQuant, which are finally presented in  Section~\ref{Proposed Techniques}.

  \subsection{Datasets for Training and Validation}
  \label{sec:datasets}
    The first two datasets contain patches extracted from the textures presented in Figure~\ref{fig_texture_datasets}, which belong to the Brodatz dataset \citep{Valkealahti:1998hl}. The built the datasets using the following methodology: First, we separate each image in half and then use the left half to create the 500 training patches and the right half to create the 500 test patches. These patches are subsets of each image containing $12\times12$ pixels. Next, for each patch we stack its 12 columns and then normalize the resulting vector to have $\ell_2$ norm equals to $1$. As in \citep{Fawzi:2014gl}, the first task consisted in discriminating test patches from the images \emph{bark} and \emph{woodgrain}, and the second task consisted in discriminating patches from the images \emph{pigskin} and \emph{pressedcl}. For future reference, we named the first task as bark\_woodgrain and the second task as pigskin\_pressedcl.

    \begin{figure}[h]
      \centering
      \subfloat[\emph{bark}]{\includegraphics[width=0.2\textwidth]{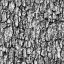}} \hspace{0.5em}
      \subfloat[\emph{woodgrain}]{\includegraphics[width=0.2\textwidth]{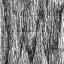}} \hspace{0.5em}
      \subfloat[\emph{pigskin}]{\includegraphics[width=0.2\textwidth]{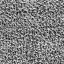}} \hspace{0.5em}
      \subfloat[\emph{pressedcl}]{\includegraphics[width=0.2\textwidth]{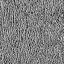}}
      \caption{Textures we used to generate the first two binary datasets.}
      \label{fig_texture_datasets}
    \end{figure}

    The third binary dataset was built using a subset of the CIFAR-10 image dataset \citep{Krizhevsky:2009tr}. This dataset contains 10 classes of 60\,000 $32\times32$ tiny RGB images, with 50\,000 images in the training set and 10\,000 in the test set. Each image has 3 color channels and it is stored in a vector of $32\times32\times3 = 3\,072$ positions. The dataset we used was the subset formed by the images labeled as \emph{deer} or \emph{horse}. 

    The first multiclass dataset was the MNIST dataset \citep{LeCun:1998we}, which contains 70\,000 images of handwritten digits of size $28\times28$ distributed in 60\,000 images in the training set and 10\,000 images in the test set. As in \citep{Fawzi:2014gl}, all images have zero-mean and $\ell_2$ norm equals to $1$. 

    The last task consisted in the classification of all 10 classes from the CIFAR-10 image dataset.

  \subsection{Theoretical Results on Computational Cost Reduction}
    \label{Theoretical Results on Computational Cost Reduction}

     For the purpose of brevity, we coined the term \emph{powerize} to concisely describe the operation of approximating each value from a set of values to its respective closest power of 2.

    \begin{theorem}
      \label{theorem:powerize upper bound}
      The relative distance ${R(x)}$ between any real scalar ${x}$ and its powerized version ${P_2(x)}$, defined by ${R(x)=\abs{P_2(x)-(x)}/x}$, is upper bounded by $1/3$.
    \end{theorem}
    \begin{proof}
      Let $2^{n} \le x \le 2^{n+1}$, $n \in \mathbb{Z}$ and $d_{P_2}(x) = \abs{P_2(x)-(x)}$ be the distance between $x$ and its powerized version. The distance $d_{P_2}(x)$ is maximum when $x = x_m = \frac{1}{2} \, (2^{n+1} + 2^{n}) = 2^{n-1} \, 3$, which is the middle point between both closest power of 2.
      
      Therefore, the distance $d_{P_2}(x_m) = \abs{x_m - 2^n} = \abs{2^{n-1} \, 3 - 2^n} = \abs{2^{n-1} \, (3 - 2)} = \abs{2^{n-1}} = \frac{x_m}{3}$, and so the maximum relative distance between $x$ and its powerized version is $R(x) = d_{P_2}(x_m)/x_m$, which is equal to $1/3$.
    \end{proof}

    We now show how the classification accuracy on the test is influenced by small variations introduced in the entries of the model $(\mathbf{D}, \mathbf{w})$. Using the datasets bark\_woodgrain and pigskin\_pressedcl described in Section~\ref{sec:datasets}, we trained an initial model $(\mathbf{D}, \mathbf{w})$, with 50 atoms, and created 50 versions $(\mathbf{D},\mathbf{w})^i$, $i = 1, 2, \dotsc, 50$ using the following steps. Each model $(\mathbf{D},\mathbf{w})^i$ were built by multiplying the entries of the initial model $(\mathbf{D}, \mathbf{w})$ by a random value chosen from the uniform distribution on the open interval $(1 - d_i, 1 + d_i)$, where $d_i \in \{0.02, 0.04, 0.06, \dotsc, 1\}$. Next, we evaluated all models on the test set.

    To get a better estimate of the classification accuracy of each model, we performed the above steps ten times on different initial models $(\mathbf{D}, \mathbf{w})$ trained using different initial values. The results, shown in Figure~\ref{fig:noise}, indicate a clear trade-off between the classification accuracy and how far the entries of $(\mathbf{D},\mathbf{w})^i$ are displaced from the corresponding entries of the original models $(\mathbf{D}, \mathbf{w})$.

    \begin{figure}[ht]
      \centering
      \subfloat[bark\_woodgrain]{\includegraphics[width=0.45\textwidth]{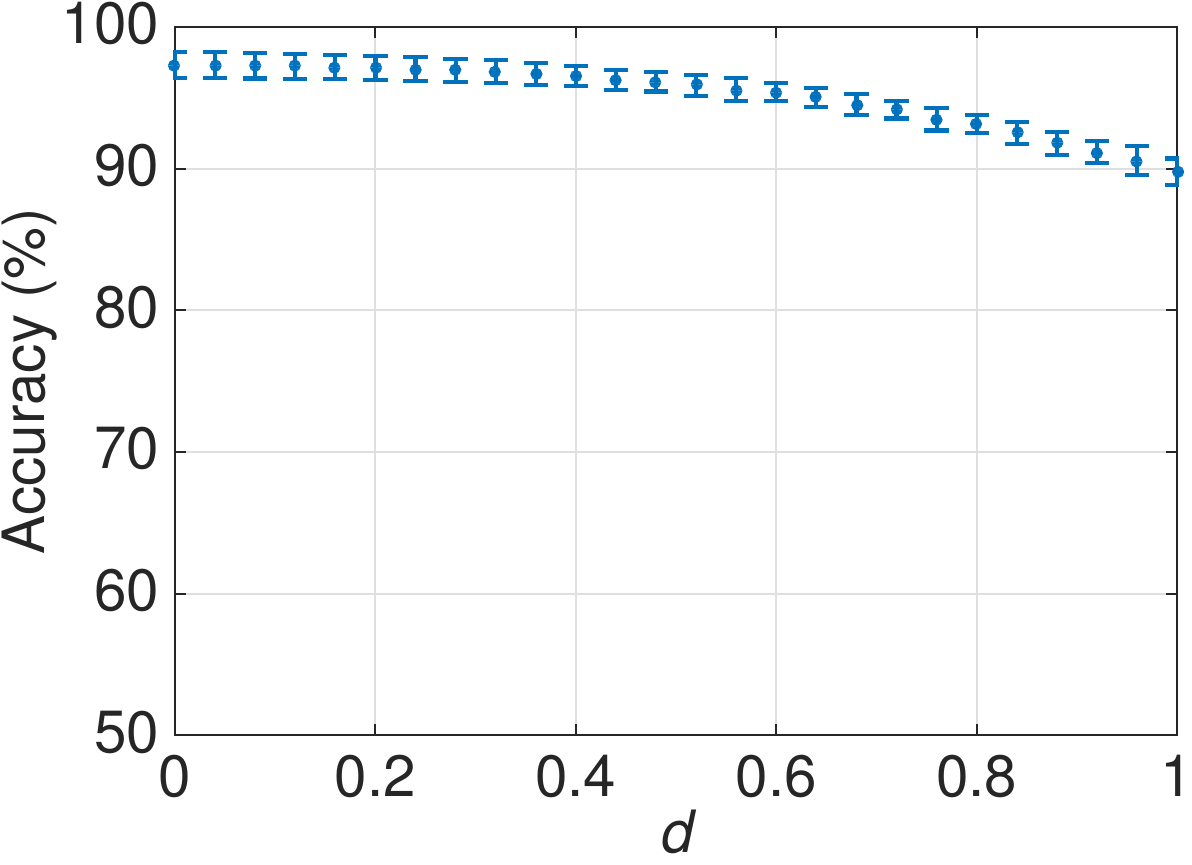}} \hspace{1em}
      \subfloat[pigskin\_pressedcl]{\includegraphics[width=0.45\textwidth]{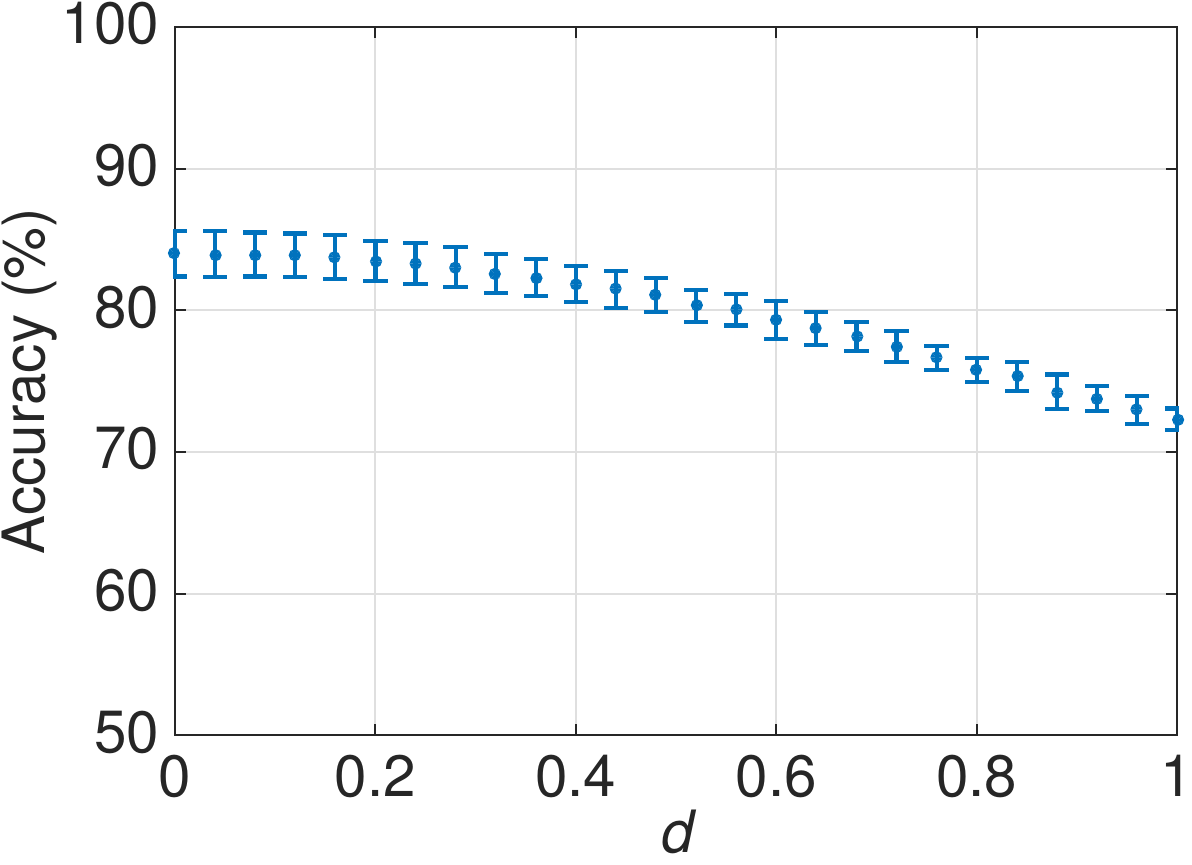}}
      \caption[]{Impact on the classification accuracy when the entries of the dictionary $\mathbf{D}$ and classifier $\mathbf{w}$ are randomly modified up to a certain level $d$.}
      \label{fig:noise}
\end{figure}
    
    \begin{hypothesis}
      \label{hypothesis:powerize}
      The model $(\mathbf{D},\mathbf{w})$ can be powerized at the cost of a slight classification accuracy decrease.
    \end{hypothesis}

    It is worth noting that the Theorem~\ref{theorem:powerize upper bound} guarantees an upper bound of $1/3$ for the relative distance between any real scalar $x$ and its powerized version. Therefore, it is reasonable to hypothesize that the classification accuracy using the powerized pair $(\mathbf{D},\mathbf{w})_{power}$ is no worse than using $(\mathbf{D},\mathbf{w})^i$, when $d_i = 1/3$, shown in Figure~\ref{fig:noise}. To support this hypothesis, we performed another simulation with the datasets bark\_woodgrain and pigskin\_pressedcl. for each dataset, we trained 10 models $(\mathbf{D},\mathbf{w})^i$ on different random versions of the training set and evaluated them and their respective powerized versions $(\mathbf{D},\mathbf{w})^i_{power}$ on the test set. Regarding the bark\_woodgrain dataset, the original model accuracy were $97.33\%\,(0.93)$ and the powerized model accuracy were $97.00\%\,(1.06)$. As for the pigskin\_pressedcl, the original model accuracy were $84.00\%\,(1.61)$ and the powerized model accuracy were $82.65\%\,(1.26)$.

    \begin{theorem}
      \label{theorem_raw_signals}
      Let $\X_{int}$ be a training set formed integer valued vectors and $\X$ be its normalized version with norm $\ell_2 = 1$, where the model $(\mathbf{D},\mathbf{w})$ is trained on. The classification accuracy of the both raw signals $\X_{int}$ and normalized signals $\X$ are exactly the same when the sparsity parameter $\alpha$ in~\emph{(\ref{eq:LAST_classification_test_time})} is $\alpha = \norm{\x_{int}}_2$ for each $\x_{int} \in \X_{int}$.
    \end{theorem}

    \begin{proof}
      Let $\mathbf{x}_{int}$ and $\mathbf{x}$ be respectively a raw vector from the test set and its normalized version, with $\norm{\mathbf{x}}_2 = 1$. Let also $(\mathbf{D},\mathbf{w})$ be the model trained with $\alpha = 1$. Therefore, the extracted features are $\mathbf{f} = \mathbf{D^\top} \mathbf{x} = \mathbf{D^\top} \frac{\mathbf{x}_{int}}{\norm{\mathbf{x}_{int}}_2}$ and the soft-thresholded feature is $\mathbf{f}_{\alpha} = \max(0, \mathbf{f} - \alpha) = \max(0, \mathbf{D^\top} \frac{\mathbf{x}_{int}}{\norm{\mathbf{x}_{int}}_2} - 1) = \frac{1}{\norm{\mathbf{x}_{int}}_2} \max(0, \mathbf{D^\top} \mathbf{x}_{int} - \norm{\mathbf{x}_{int}}_2)$. Finally, the classification of $\mathbf{x}_{int}$ is $c = (\mathbf{w} \, \frac{1}{\norm{\mathbf{x}_{int}}_2} \max(0, \mathbf{D^\top} \mathbf{x}_{int} - \norm{\mathbf{x}_{int}}_2) > 0)$.

      As the $\ell_2$ norm of any real vector different from the null vector is always greater than 0, then $\frac{1}{\norm{\mathbf{x}_{int}}_2} > 0$, and, thus $c = (\mathbf{w} \max(0, \mathbf{D^\top} \mathbf{x}_{int} - \norm{\mathbf{x}_{int}}_2) > 0)$.

      Therefore, as $\mathbf{x} = \mathbf{x}_{int} / \norm{\mathbf{x}_{int}}_2$, the expressions $c = (\mathbf{w} \max(0, \mathbf{D^\top} \mathbf{x} - \alpha) > 0)$, with $\alpha = 1$, and $c = (\mathbf{w} \max(0, \mathbf{D^\top} \mathbf{x}_{int} - \alpha) > 0)$, with $\alpha = \norm{\mathbf{x}_{int}}_2$ are equivalent.
    \end{proof}

    \begin{empirical_evidence}
      Forcing the dictionary $\mathbf{D}$ to be sparse by hard thresholding its entries up to a certain level will decrease its dynamic range and thus reduce the number of bits necessary to compute $\mathbf{D}^\top \mathbf{X}$ at the cost of a slight classification accuracy decrease.
    \end{empirical_evidence}

    We hypothesized that forcing $\mathbf{D}$ to be sparse would decrease its dynamic range with no substantial decrease of its classification accuracy. To support our hypothesis we performed another simulation with the datasets bark\_woodgrain and pigskin\_pressedcl. For each dataset, we trained a model $(\mathbf{D},\mathbf{w})$ and created 14 versions of it by hard-thresholding the entries of $\mathbf{D}$ using 14 threshold values linearly spaced between $0$ and $4$. Subsequently, we divided each element of the hard-thresholded dictionary $\mathbf{D}_{t}$ by the lowest value from $\abs{\mathbf{D}_{t}}$ that is different from 0. 

    Finally, we evaluated all resulting models on the test set. For a better estimate of the classification accuracy, we performed the above steps on 10 models $(\mathbf{D},\mathbf{w})$ trained on different random versions of the training set and computed their average. As shown in Figure~\ref{fig:hard_threshold_accuracy}(a), the first threshold different from zero already reduces the bit precision of $\mathbf{D}_{t}$ to less than half of the original while slightly decreasing its classification accuracy. Also, the third threshold different from 0 shown in Figure~\ref{fig:hard_threshold_accuracy}(b) almost maintains the same classification accuracy while reducing its dynamic range to less than half of the original.
    
    \begin{figure}[ht]
      \centering
      \subfloat[bark\_woodgrain]{
        \begin{minipage}{0.45\columnwidth}
          \includegraphics[width=1\textwidth]{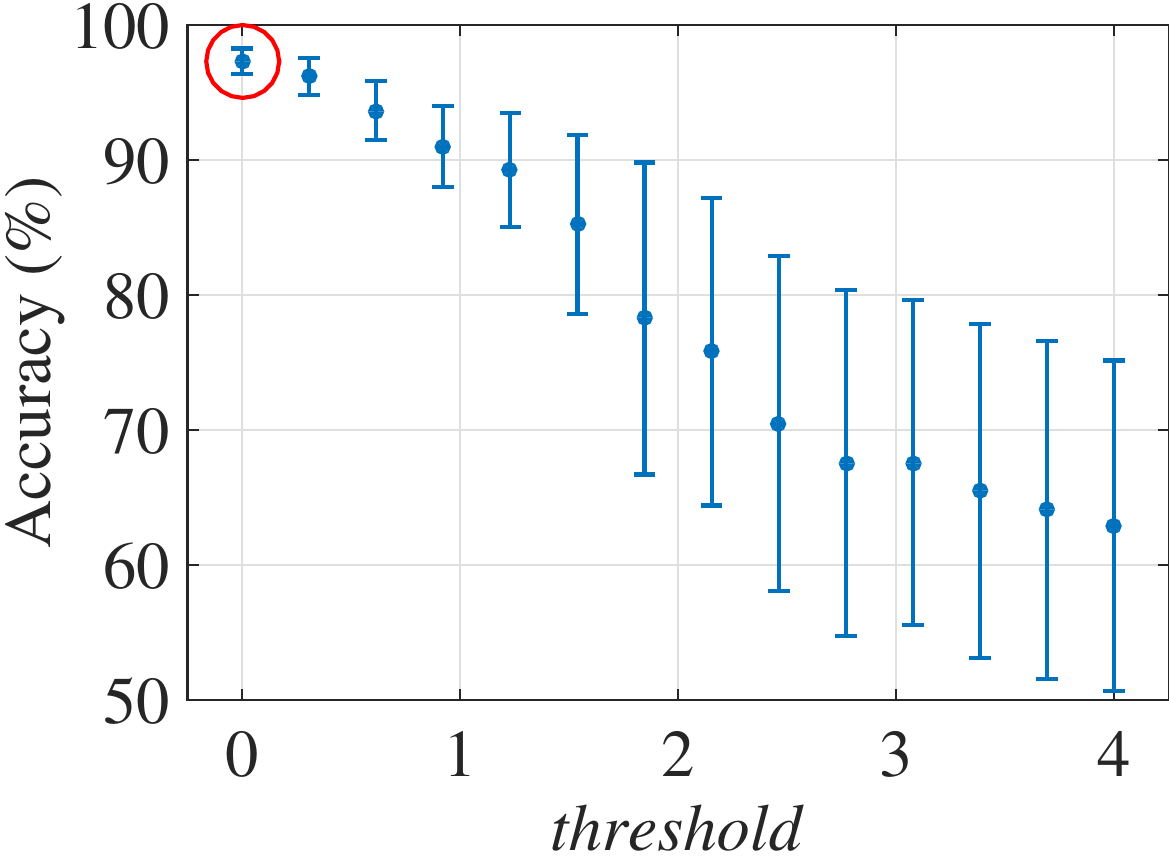} \\
          \includegraphics[width=0.968\textwidth,right]{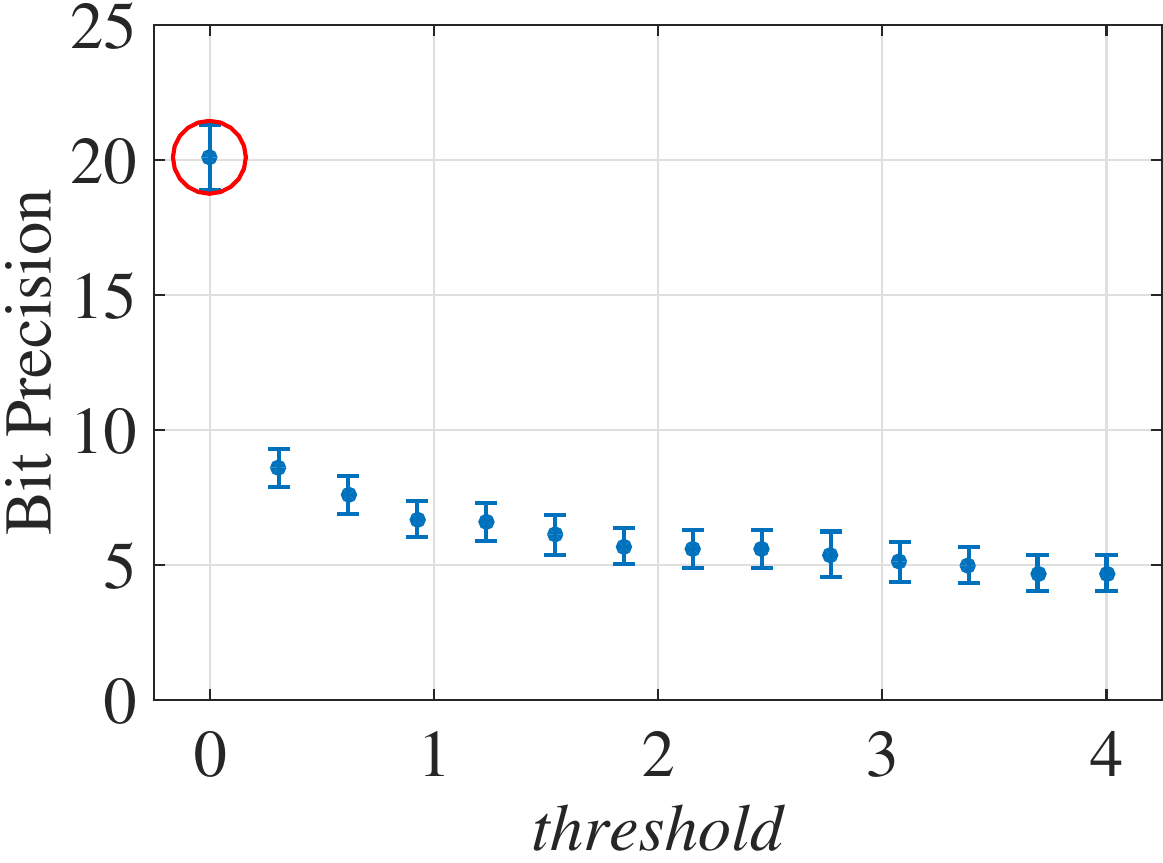}
        \end{minipage}
      } \hspace{1em}
      \subfloat[pigskin\_pressedcl]{
        \begin{minipage}{0.45\columnwidth}
          \includegraphics[width=1\textwidth]{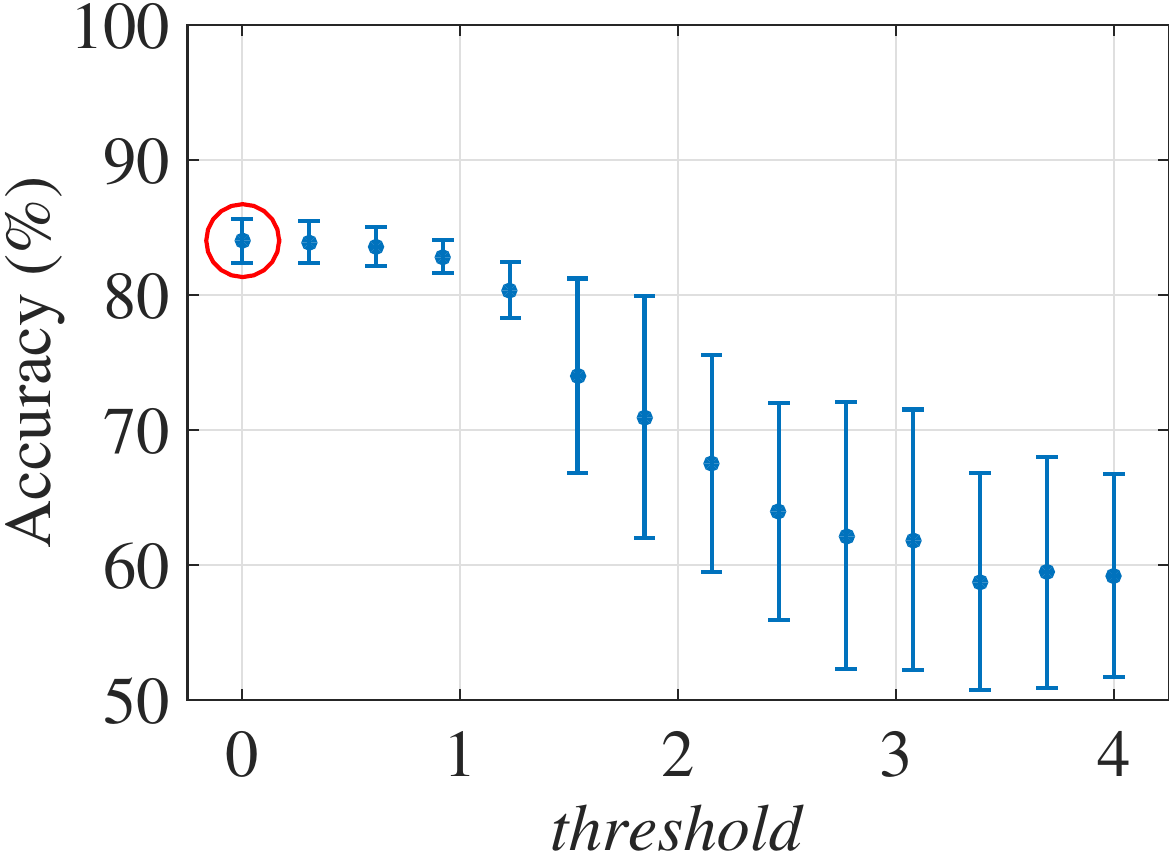} \\
          \includegraphics[width=0.968\textwidth,right]{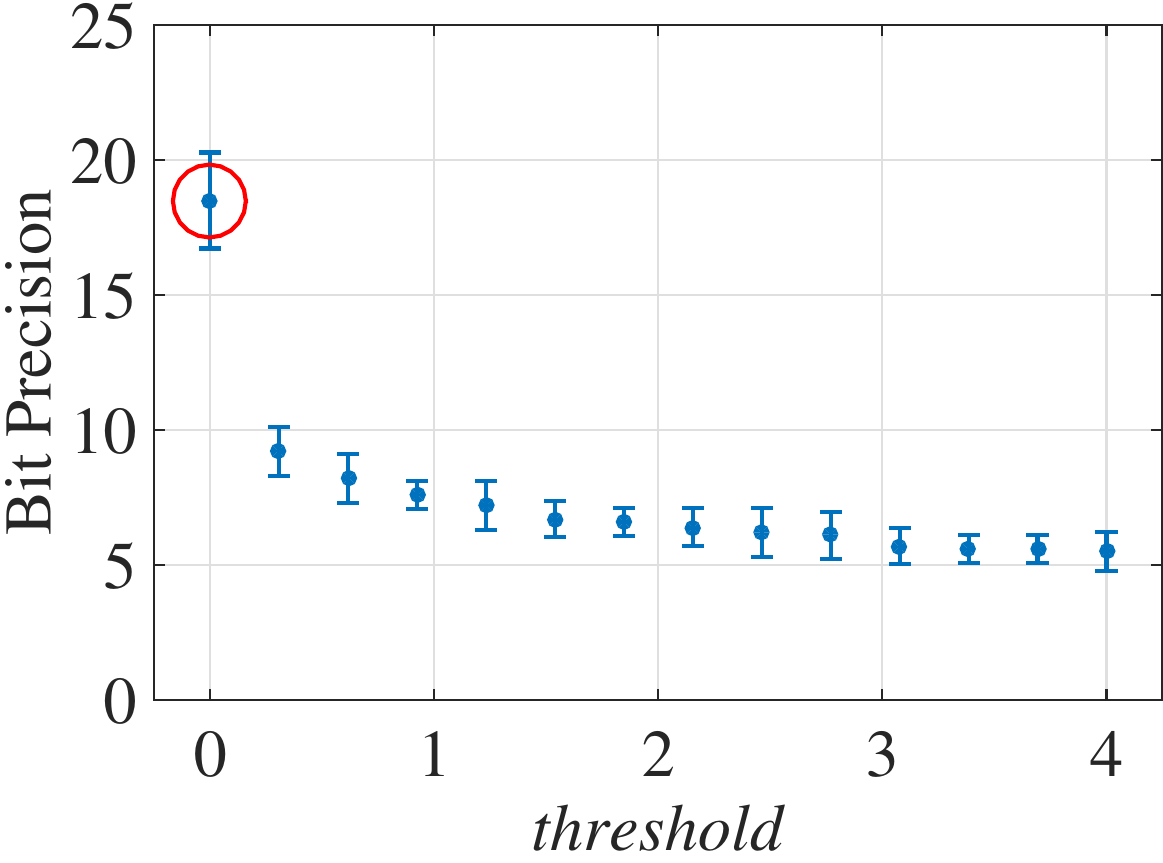}
        \end{minipage}
      } 



      \caption[]{Impact on the classification accuracy when hard threshold is used to reduce the bit precision of dictionary $\mathbf{D}$. The values shown are the average of the classification accuracy on the test set evaluated with 10 models $(\mathbf{D},\mathbf{w})$, with 50 atoms, trained with different training sets. The original results are marked with a red circle. The datasets are described in Section~\ref{sec:datasets}.}
      \label{fig:hard_threshold_accuracy}
\end{figure}

    \begin{empirical_evidence}
      Quantizing the integer valued images from the test set $\mathbf{X}_{int}$ up to a certain level will decrease the dynamic range of $\mathbf{X}_{int}$ and thus reduce the number of bits necessary to compute $\mathbf{D}^\top \mathbf{X}_{int}$ at the cost of a slight classification accuracy decrease.
    \end{empirical_evidence}

    We also hypothesized the original integer valued signals were unnecessarily over quantized and that their quantization level could be decreased while not substantially worsening the classification accuracy. To support our hypothesis, we performed another simulation with the datasets bark\_woodgrain and pigskin\_pressedcl. For each dataset, we averaged the results of one thousand runs consisting in 10 models $(\mathbf{D},\mathbf{w})$ trained using different training sets and evaluated on different quantized versions of the test set. The images from each test set $\mathbf{X}_{int}$ were quantized using levels ranging from 1 to 15. The results are shown in Figure~\ref{fig:quantization_accuracy}. Its worth noting in this figure that images from both datasets can have their bit precision reduced to 2 (Quantization level equals to 2 and 3) while having a limited decrease of the classification accuracy.

    \begin{figure}[ht]
      \centering
      \subfloat[bark\_woodgrain]{\includegraphics[width=0.45\textwidth]{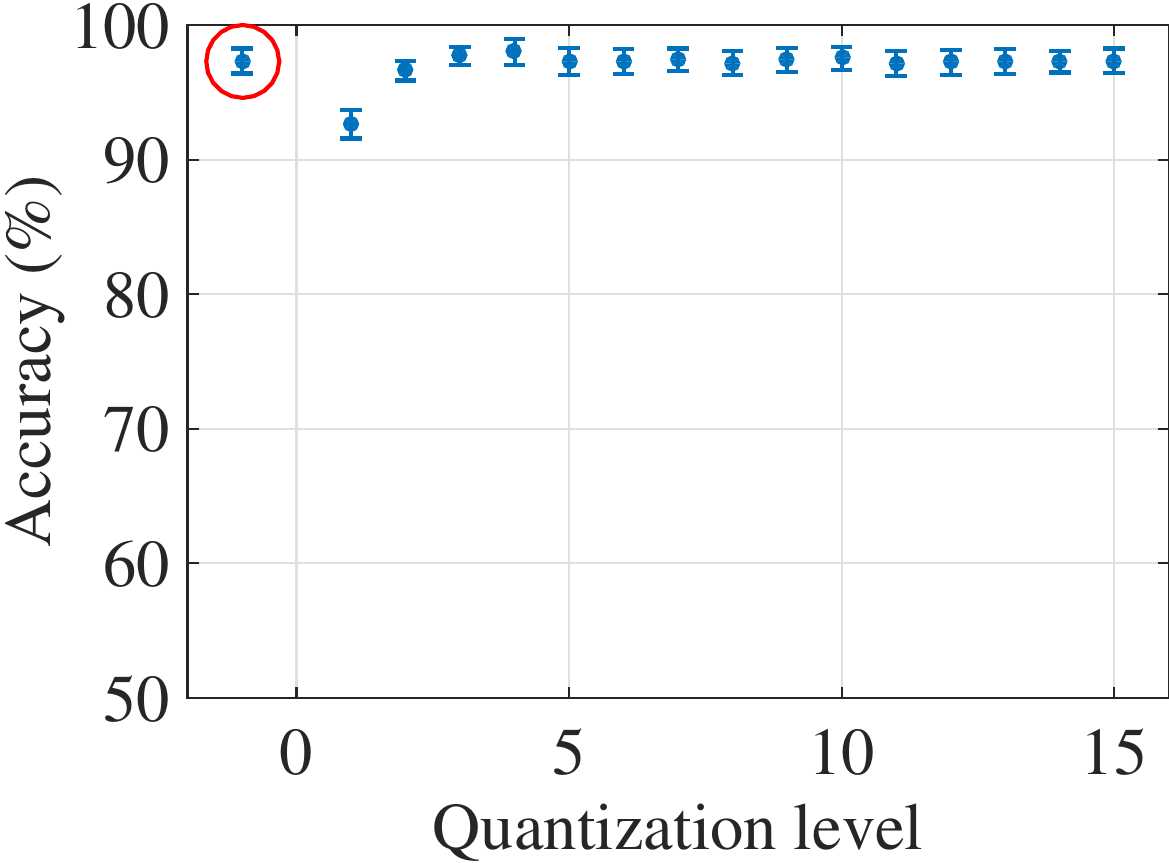}} \hspace{1em}
      \subfloat[pigskin\_pressedcl]{\includegraphics[width=0.45\textwidth]{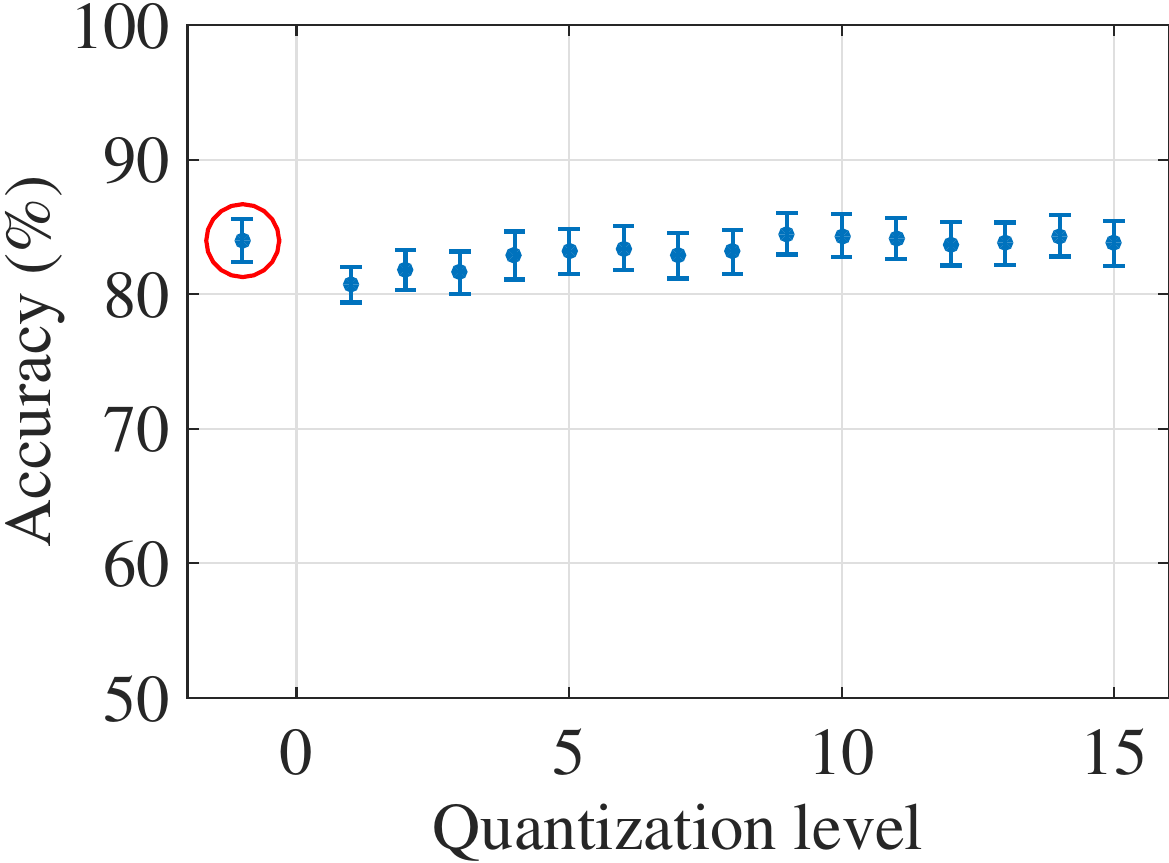}}

      \subfloat[bark\_woodgrain]{\includegraphics[width=0.45\textwidth]{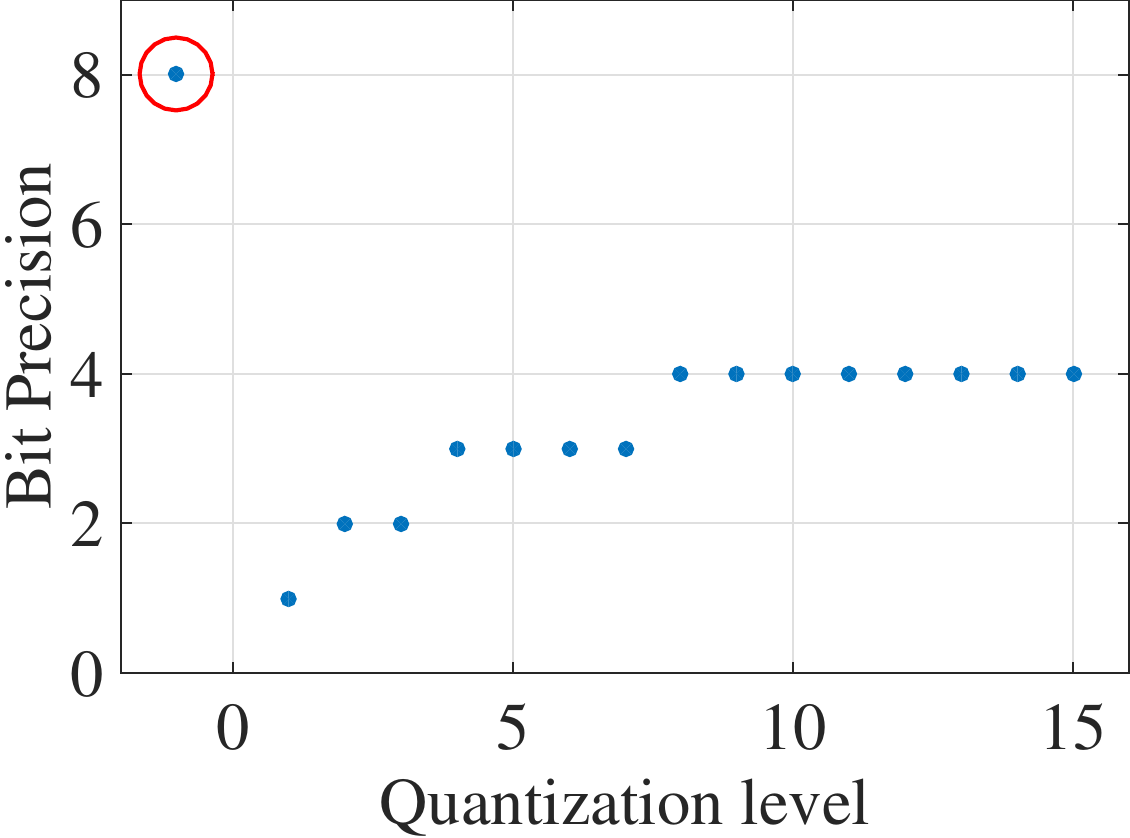}} \hspace{1em}
      \subfloat[pigskin\_pressedcl]{\includegraphics[width=0.45\textwidth]{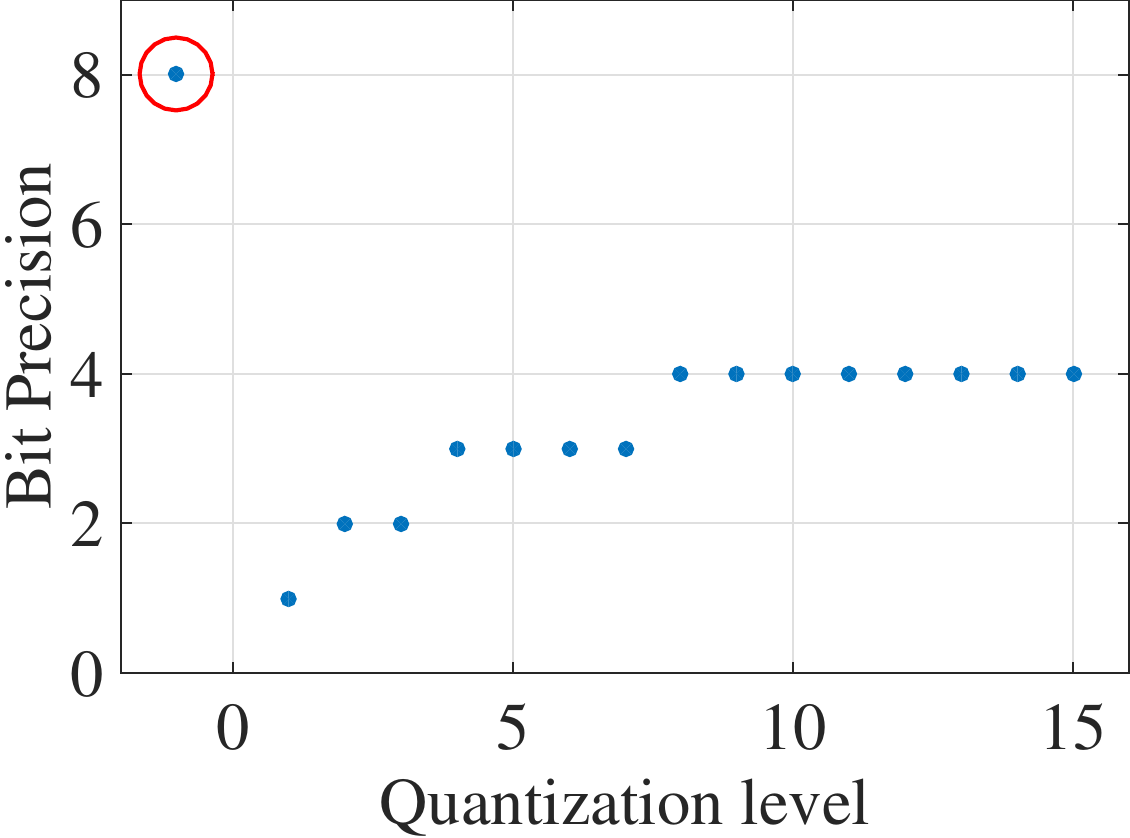}}

      \caption[]{Impact on the classification accuracy when the images of the test set are quantized up to a certain level. The original results are marked with a red circle. Note that reducing the bit precision of the test set images to as low as 2 bits does not substantially worsens the classification accuracy. These results are the average of the classification results of the test set evaluated with 10 models $(\mathbf{D},\mathbf{w})$, with 50 atoms, trained with different training sets. The datasets are described in Section~\ref{sec:datasets}.}
      \label{fig:quantization_accuracy}
\end{figure}

  \subsection{Proposed Techniques}
    \label{Proposed Techniques}
    \begin{technique}
      \label{tech:integer}
      Use signals in its raw representation (in integer) rather than their normalized version (in floating-point).
      \vspace{-4pt}
    \end{technique}

    \begin{technique}
      \label{tech:powerize}
      Powerize $\mathbf{D}$ and $\mathbf{w}$.
      \vspace{-4pt}
    \end{technique}

    \begin{technique}
      \label{tech:X_compression}
      Decrease the dynamic range of the test set $\mathbf{X}_{int}$ by quantizing the integer valued test images $\mathbf{X}_{int}$.
      \vspace{-4pt}
    \end{technique}

    \begin{technique}
      \label{tech:D_compression}
      Decrease the dynamic range of the entries of $\mathbf{D}$ by penalizing their $\ell_2$-norm during the training followed by hard-thresholding, using a trained threshold level.
      \vspace{-4pt}
    \end{technique}
    Our strategy to decrease the dynamic range of the dictionary $\mathbf{D}$ involves the addition of a penalty to the $\ell_2$ norm of its entries during the minimization of the objective function of LAST, described in (\ref{eq:objective_function_LAST}). The motivation for penalizing the ${\ell_2}$ of ${\mathbf{w}}$ and ${\mathbf{D}}$ is the fact that this can avoid solutions containing high-valued entries, which would require a representation using more bits. Also note that penalizing the ${\ell_1}$, which would seem more reasonable in terms of providing sparse dictionaries, would still allow for higher entries (even if in small numbers), which would anyway require more bits for proper quantization. The new proposed optimization problem hence becomes
    \begin{equation}
      \label{eq:objective_function_LAST_sparsify}
      \min_{\mathbf{D},\mathbf{w}} \sum^{m}_{i=1}H(y_i \mathbf{w^\top} h_{\alpha}(\mathbf{D^\top x_i})) + \frac{v}{2}\norm{\mathbf{w}}^2_2 + \frac{\kappa}{2}\norm{\mathbf{D}}^2_2,
    \end{equation}
    where $\kappa$ controls this new penalization. In Section~\ref{ell_2_penalization}, we show our proposed technique of including this penalization into general constrained optimization algorithms, followed by how we included this penalization into the difference of convex (DC) optimization algorithm used in LAST~\citep{Fawzi:2014gl}.

    After training $\mathbf{D}$ and $\mathbf{w}$ using the modified objective function (\ref{eq:objective_function_LAST_sparsify}), we apply a hard-threshold to its entries to zero out the values closer to zero. Our assumption is that these small values of $\mathbf{D}$ have little contribution on the final feature value and, thus, can be set to zero without affecting much the classification accuracy. As for the threshold value, we test the best one from all unique absolute values of $\mathbf{D}$ after it has been powerized using our Technique~\ref{tech:powerize}. As the number of unique absolute values of $\mathbf{D}$ is substantially reduced after using the Technique~\ref{tech:powerize}, the computational burden to test all possible values is greatly reduced.

  \subsection{Inclusion of an $\ell_2$ Norm Penalization Term in Dictionary Training Algorithms Based on Constrained Optimization}
    \label{ell_2_penalization}
    We show how to include a term into the objective function that penalizes potential dictionaries whose entries have larger energy values, as opposed to lower-energy dictionaries. By favoring vectors with lower energies, we may obtain dictionaries which span over narrower ranges of values. In our development, we consider the inclusion of this penalization into gradient descent (GD) methods, as many optimization problems are based on GD~\citep{Boyd:2004uz}. In our experimental evaluations, we test the proposed methods by modifying the algorithm in~\citep{Fawzi:2014gl}, which use GD to solve the optimization problem. The development in this section applies to both our modifications in~\citep{Fawzi:2014gl} and to other methods based on GD.

    Several dictionary and classifier training methods are based on constrained optimization programs such as~\citep{Fawzi:2014gl,Ravishankar:2013jj}
    \begin{equation}
      \label{eq:optimization_dictionary_classifier}
      \underset{\V,\w}{\text{min }} f(\V,\w) \text{ s.t. } g(\V,\w) = \zero,
    \end{equation}
    where:
    \begin{inparaenum}[(i)]
      \item $\V$ is an $n_1 \times 1$ vector containing the dictionary terms and ${\w}$ is an $n_2 \times 1$ vector of classifier parameters;
      \item ${f:\mathbb{R}^n\rightarrow\mathbb{R}}$, $n = n_1 + n_2$, is the cost function based on the training set;
      \item ${\zero}$ is the null vector;
      \item and ${g:\mathbb{R}^m\rightarrow\mathbb{R}}$ is a function representing ${m}$ scalar equality constraints.
    \end{inparaenum}
    Some methods also include inequality constraints.

    In order to penalize the total energy associated to the dictionary entries, we can replace any problem of the form~\eqref{eq:optimization_dictionary_classifier} by
    \begin{equation}
      \label{eq:optimization_dictionary_classifier_penalization}
      \underset{\V,\w}{\text{min }} f(\V,\w) + \kappa \, \frac{1}{2} \norm{\V}_{2}^2 \text{ s.t. } g(\V,\w) = \zero,
    \end{equation}
    where ${\kappa > 0}$ is a penalization weight.

    Iterative methods are commonly used to solve constrained optimization problems~\citep{Boyd:2004uz} such as \eqref{eq:optimization_dictionary_classifier_penalization}. They start with an initial value ${\x^{0}=[\V^{0}\,\,\,\w^{0}]^T}$ for ${\x=[\V\,\,\,\w]^T}$, which is iterated to generate a supposedly convergence sequence ${\x^{(n)}}$ satisfying
    \begin{align}\label{eq:iteration}
      \x^{(n+1)} = \x^{(n)} + \xi\Delta\x^{(n)},\,\,\, \forall\,\,n\geq 0,
    \end{align}
    where ${\xi}$ is the step size and ${\Delta\x^{(n)}=[\Delta\V^{(n)}\,\,\,\Delta\w^{(n)}]}$ is the step computed based on the particular iterative method.

    We consider the GD method, where computing ${\Delta\x^{(n)}}$ requires evaluating the gradient of a dual function associated with the objective function and the constraints~\citep{Boyd:2004uz}. Specifically, the Lagrangian ${L(\V,\w)}$ is an example of a dual function, thus having a local maximum that is a minimum of the objective function at a point that satisfies the constraints. For problems~\eqref{eq:optimization_dictionary_classifier} and~\eqref{eq:optimization_dictionary_classifier_penalization}, the Lagrangian functions are given respectively by
    \begin{gather}
      \label{eq:originalLagrangian}
      L(\V,\w,\vlambda) = f(\V,\w) +\vlambda^Tg(\V,\w) \text{ and} \\
      \label{eq:modifiedLagrangian}
      \hat{L}(\V, \w, \vlambda) = f(\V, \w) + \vlambda^Tg(\V, \w) + \kappa \, \frac{1}{2}\norm{\V}_{2}^2,
    \end{gather}
    with ${\vlambda}$ the vector of ${m}$ Lagrange multipliers.

    Our first objective regarding solving the modified problem~\eqref{eq:optimization_dictionary_classifier_penalization} is to compute the gradient of ${\hat{L}(\V, \w, \vlambda)}$ in terms of the gradient of ${L(\V, \w, \lambda)}$, so as to show how a problem that solves~\eqref{eq:optimization_dictionary_classifier} can be modified in order to solve~\eqref{eq:optimization_dictionary_classifier_penalization}.
%
%
    By comparing~\eqref{eq:originalLagrangian} and \eqref{eq:modifiedLagrangian}, and by defining ${\nabla_\v g}$ as the gradient of any function ${g}$ with respect to vector ${\v}$ as we compute the gradients, we obtain
%
%
    \begin{gather}
      \label{eq:gradientd}
      \nabla_{\V}\hat{L}(\V,\w,\vlambda) = \nabla_{\V}{L}(\V,\w,\vlambda) + \kappa\V, \\
      \label{eq:gradientw}
      \nabla_{\w}\hat{L}(\V,\w,\vlambda) = \nabla_{\w}{L}(\V,\w,\vlambda) \text{, and} \\
      \label{eq:gradientl}
      \nabla_{\vlambda}\hat{L}(\V,\w,\vlambda) = \nabla_{\vlambda}{L}(\V,\w,\vlambda).
    \end{gather}
    
    Equations~\eqref{eq:gradientd}, \eqref{eq:gradientw}, and \eqref{eq:gradientl} show how we modify the estimated gradient in any GD method (such as LAST~\citep{Fawzi:2014gl}) in order to penalize the range of the dictionary entries, and thus try to force a solution with a narrower range. Note that only the gradient with respect to the dictionaries is altered.

\section{Simulations}
  \label{Simulations}
  In this section, we evaluate how our techniques affect the accuracy of LAST on the same datasets used in \citep{Fawzi:2014gl}. For this, we performed many simulations using the datasets presented in Section~\ref{sec:datasets} and compared their classification accuracy/error and classification bit precision, that is, minimum number of bits necessary to perform the classification. We present in Section~\ref{Choice of Classifier Parameters} the parameters we chose to generate these models and, at last, the analysis of the results we obtained comes in Section~\ref{Results and Analyses}.

  \subsection{Choice of Classifier Parameters}
    \label{Choice of Classifier Parameters}
    For all tested datasets, we fixed the parameter $\kappa = \{4, 8, 10, \dotsc, 20\}*10^{-3}$ and let $z\_threshold$ assume all unique values of the powerized version of $\mathbf{D}_{power}$, i.e., after applying the Technique~\ref{tech:powerize}. As the number of unique values of $\mathbf{D}_{power}$ is substantially lower than the ones of $\mathbf{D}$, the necessary computational burden to test all valid thresholds is low. Also, we fixed the quantization parameter ${\tt quanta} = \{1,2,\dotsc,10\} \cup \{31,127\}$. The choice of these parameter values was empirically based on a previous run of all simulations. As for the parameters in LAST, we used the same used in \citep{Fawzi:2014gl}. We direct the reader to \citep{Fawzi:2014gl} for further understanding of the parameters and their values used in LAST.

  \subsection{Model Selection}
    \label{Model Selection}
    Due to the large number of parameter combinations of both Technique~\ref{tech:X_compression} and Technique~\ref{tech:D_compression} our simulations generate many different models with classification accuracy/error and classification bit precision. To select the best model, that is, the best combination of the parameters $\kappa$, $z\_threshold$, and $quanta$ we relied on the classification accuracy on a separate data set. Also, we created the parameter $\gamma$ to control the trade-off between the classification accuracy and the classification bit precision. We used $\gamma = 0.001$ and the following steps for the model selection:
    \begin{inparaenum}[(i)] 
      \item First, we used $80\%$ of the training set to train the models ($\mathbf{D}$ and $\mathbf{w}$) and used the remaining $20\%$ to estimate the best combination of the parameters $\kappa$, $z\_threshold$, and $quanta$.
      \item Let $\mathcal{M}$ be the set of models trained with all combinations of the parameters $\kappa$, $z\_threshold$, and $quanta$. Also, let $\mathcal{R}=\mathcal{M}(\mathbf{X})$ be the set of the classification results of the training set $\mathbf{X}$ using the models $\mathcal{M}$ and $best\_acc$ be the best training accuracy from $\mathcal{R}$.
      \item From $\mathcal{M}$, we create the subset $\mathcal{M}_{\gamma}$ that contains the models with results $\mathcal{R}_{\gamma} = \mathcal{R}[\text{accuracy} >= (1 - \gamma) \, best\_acc]$. 
      \item From $\mathcal{M}_{\gamma}$, we create a new subset $\mathcal{M}_{bits}$ with results $\mathcal{R}_{bits} = \mathcal{R}_{\gamma}[\text{number of bits} == lowest\_num\_bits]$, where $lowest\_num\_bits$ is the lowest number of bits necessary for the computation of $\mathbf{D^{\top}X}$.
      \item From $\mathcal{R}_{bits}$, we finally choose the model $\mathcal{M}_{best}$ such that the result $\mathcal{R}_{best} = \mathcal{R}_{bits}[\text{sparsest representation of }\mathbf{X}]$.
    \end{inparaenum}    

    It is worth noting that the traditional rule of thumb of using $2/3$ of the dataset to train and $1/3$ to test is a safe way of estimating of the true classification accuracy when the classification accuracy on the whole dataset set is higher than $85\%$ \citep{Dobbin:2011dv}. Nevertheless, as we are solely reserving part of the training set for the selection of the best parameters values, and not for the estimation of the true classification accuracy, we opted for the more conservative proportion of $80\%$ to train our models. This has the advantage of lowering the chance of missing an underrepresented training set sample. Moreover, the last step in our model selection algorithm selects the model that produces the sparsest signal representation, as it leads to models that generalize better \citep{Bengio:2013bu}.

  \subsection{Results and Analyses}
    \label{Results and Analyses}
    In this section, the \emph{original} results are the ones from the classification of the test set using the model built with the original LAST algorithm. Conversely, the \emph{proposed} results are the ones obtained from the classification of the test set using the best model $\mathcal{R}_{best}$ built for each dataset. The best model $\mathcal{R}_{best}$ is the one selected using the methodology presented in Section~\ref{Model Selection}. 

    We show the results of our simulations on the binary tasks in Figure~\ref{fig:binary_dict_size_vs_acc}. As shown on the bottom of Figures~\ref{fig:binary_dict_size_vs_acc}(a), \ref{fig:binary_dict_size_vs_acc}(b), and \ref{fig:binary_dict_size_vs_acc}(c), our techniques do not substantially decrease the original classification accuracy. At the same time, our techniques considerably reduce the number of bits necessary to perform the multiplication $\mathbf{D}^{\top}\mathbf{X}$, as shown on the top of Figures~\ref{fig:binary_dict_size_vs_acc}(a), \ref{fig:binary_dict_size_vs_acc}(b), and \ref{fig:binary_dict_size_vs_acc}(c).

    One can note the original results in Figures~\ref{fig:binary_dict_size_vs_acc}(a) and \ref{fig:binary_dict_size_vs_acc}(c) are lower than the ones presented in~\citep{Fawzi:2014gl}. Differently from their work, we used completely disjoint training and test sets (with no overlap) to allow a better estimation of the true classification accuracy.

    \begin{figure}[ht]
      \centering
      
      \centering
      \subfloat[bark\_woodgrain]{
        \begin{minipage}{0.32\columnwidth}
          \includegraphics[width=1\textwidth]{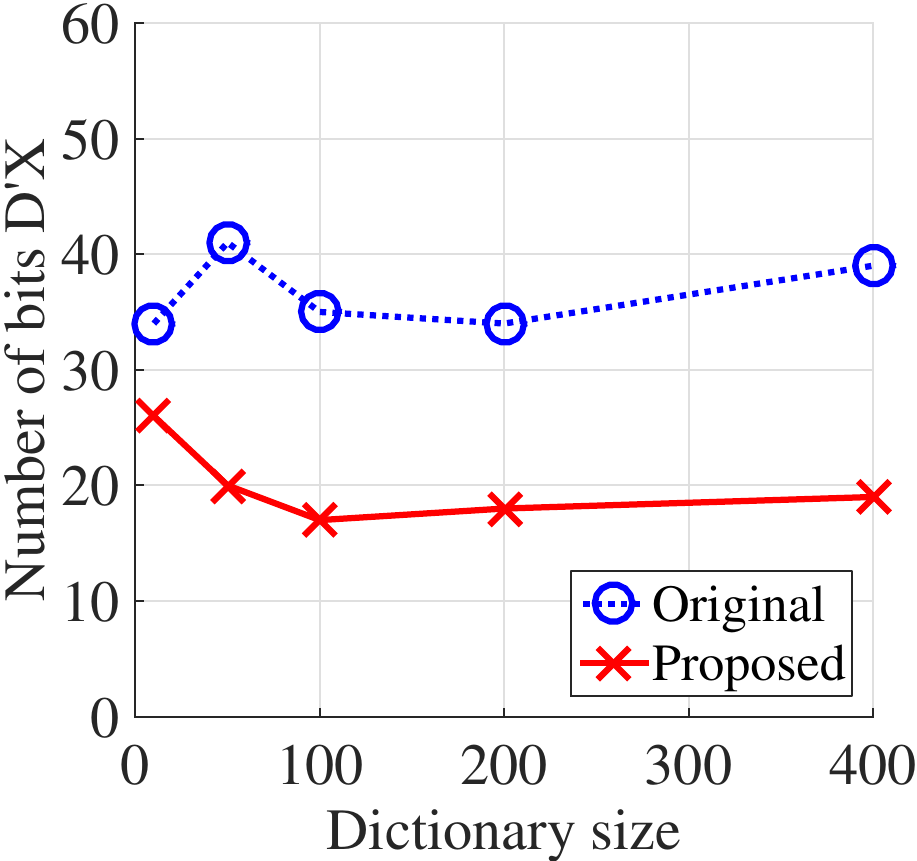} \\
          \includegraphics[width=1.037\textwidth,right]{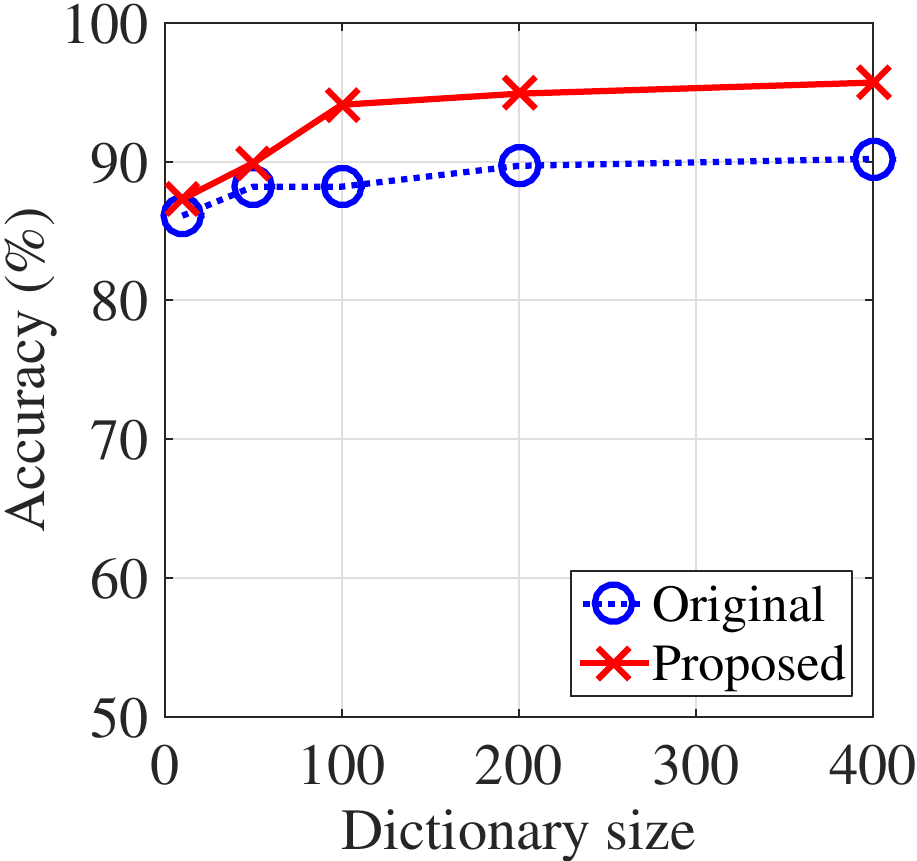}
        \end{minipage}
      } 
      \subfloat[pigskin\_pressedcl]{
        \begin{minipage}{0.32\columnwidth}
          \includegraphics[width=1\textwidth]{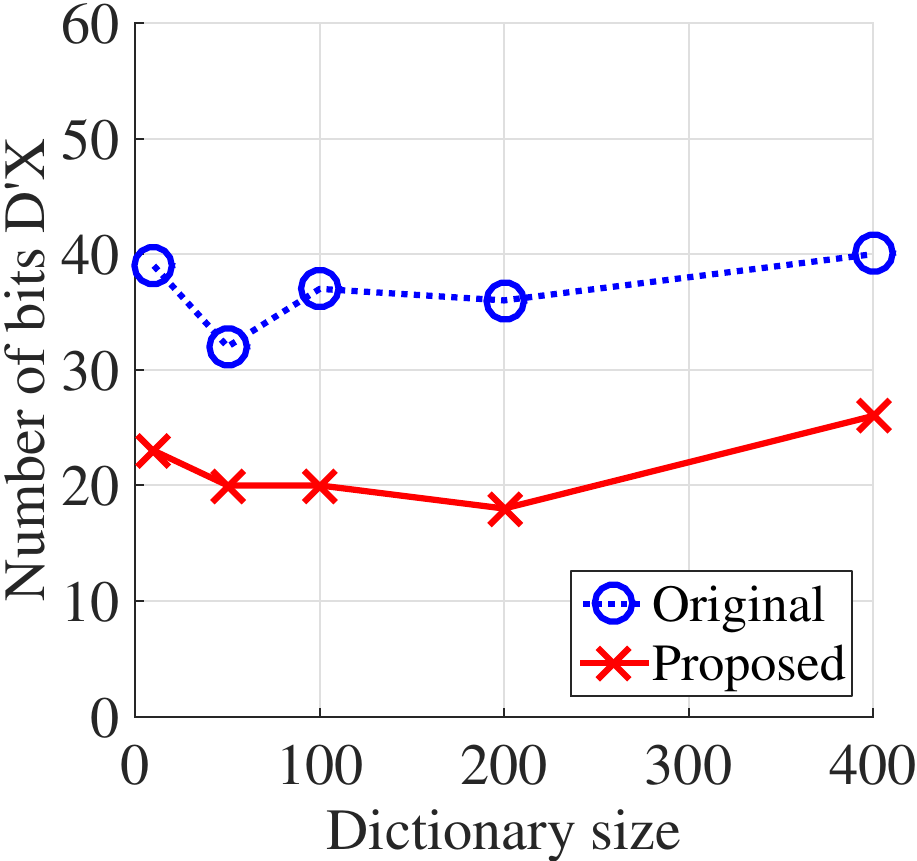} \\
          \includegraphics[width=1.037\textwidth,right]{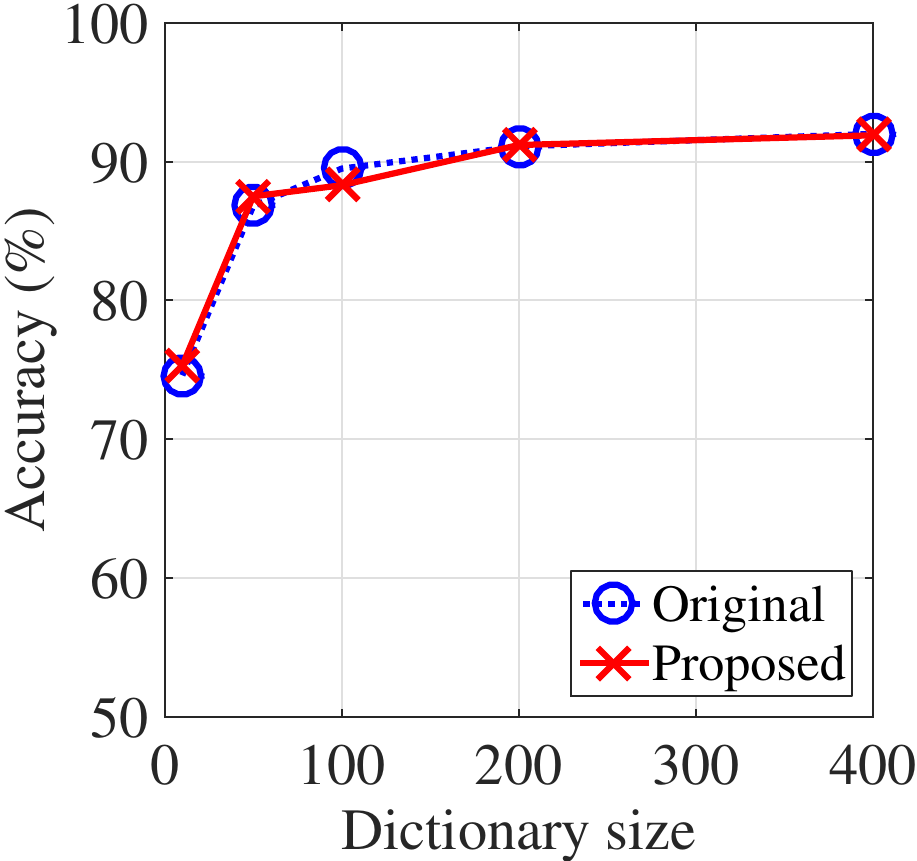}
        \end{minipage}
      } 
      \subfloat[CIFAR-10 deer\_horse]{
        \begin{minipage}{0.32\columnwidth}
          \includegraphics[width=1\textwidth]{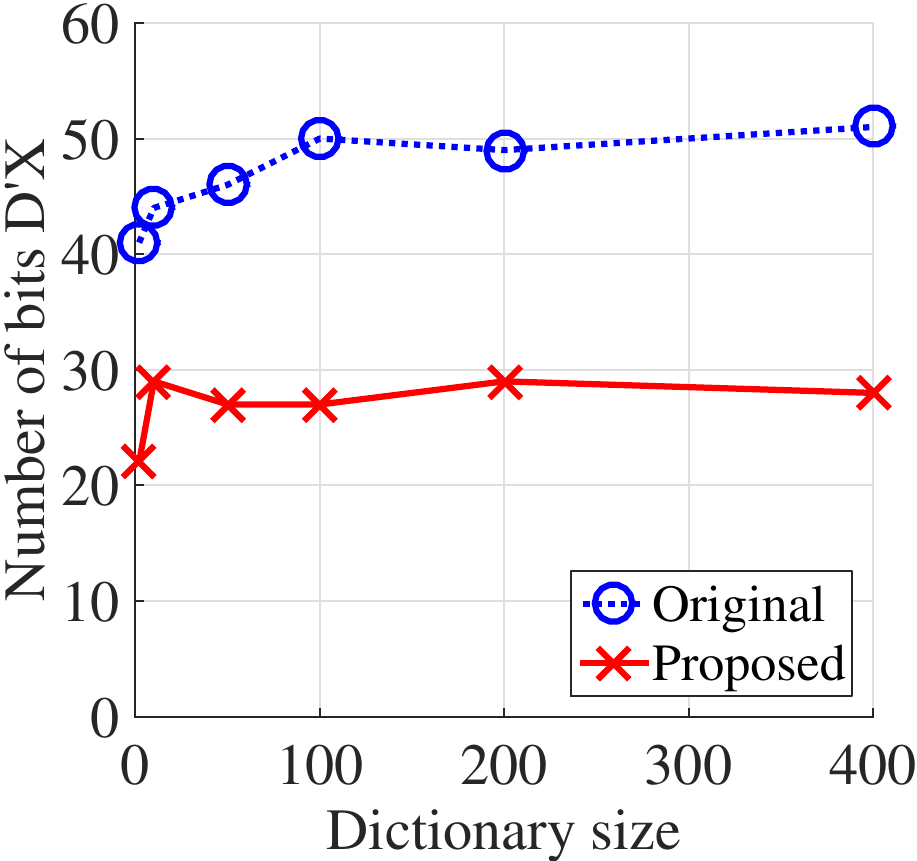} \\
          \includegraphics[width=1.037\textwidth,right]{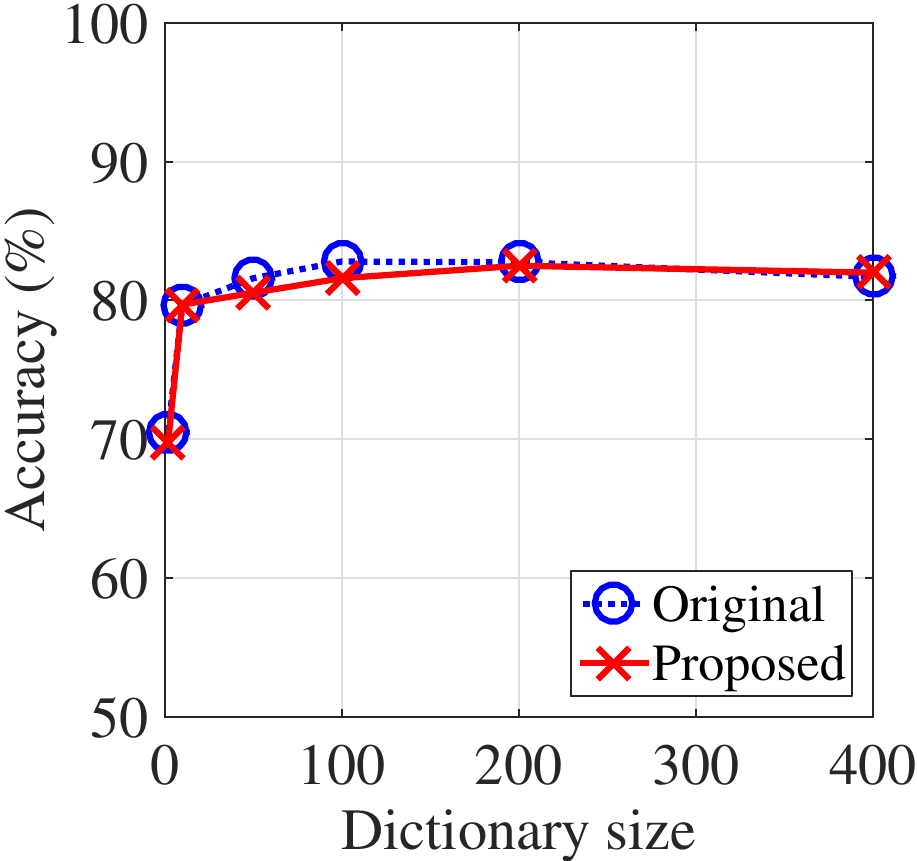}
        \end{minipage}
      } 



      \caption{Comparison of the results using the original LAST algorithm and our proposed techniques. Regarding the classification at test time, these figures show for each dataset the trade-off between the necessary number of bits (top) and the classification accuracy (bottom). Our approach reduces the necessary number of bits to almost half of the original formulation at the cost of a slight classification accuracy decrease. The datasets are described in Section~\ref{sec:datasets}.}
      \label{fig:binary_dict_size_vs_acc}
\end{figure}

    Table~\ref{table:table_mnist_cifar_10_error_and_num_bits_D_X} contains the results of the simulations on the tasks MNIST and CIFAR-10. The original results we obtained for both large datasets have a slightly higher classification error than the ones reported in \citep{Fawzi:2014gl}. We hypothesize that this is caused by the random nature of LAST for larger datasets, where each GD is optimized for a small portion of the data called mini-batch, which is randomly sampled from the training set. Moreover, we trained $\mathbf{D}$ and $\mathbf{w}$ using $4/5$ of the training set used in \citep{Fawzi:2014gl} and this may negatively affect the generalization power of the dictionary and classifier. 

    Note that our techniques resulted in an increase of the classification error on both MNIST and CIFAR-10 tasks. Nevertheless, our techniques reduced the number of bits necessary to run the classification at test time. Again, this dynamic range reduction is highly valuable for applications on FPGA.


    \begin{table}[ht]
      \centering
      \small
        \caption[...]{Comparison between the original and the proposed results regarding the classification error and number of bits necessary to compute the matrix-vector multiplication $\mathbf{D}^{\top}\mathbf{X}$ of the sparse representation.}
      \begin{tabular}{lcccc}  
        \toprule 
        \ExecuteMetaData[my_files/table_mnist_cifar_10_error_and_num_bits_D_X.tex]{header}
        \midrule 
        \ExecuteMetaData[my_files/table_mnist_cifar_10_error_and_num_bits_D_X.tex]{data}
        \bottomrule 
        \label{table:table_mnist_cifar_10_error_and_num_bits_D_X}
      \end{tabular}
\end{table}

    The results we presented in this section indicate the feasibility of using integer operations in place of floating-point ones and use bit shifts instead of multiplications with a slight classification accuracy decrease. These substitutions reduce the computational cost of classification at test time in FPGAs, which is important in embedded applications, where power consumption is critical. Moreover, our techniques reduce almost in half the number of bits necessary to perform the most expensive operation in the classification, the matrix-vector multiplication $\mathbf{D}^{\top}\mathbf{X}$. This was a result of the application of both Technique~\ref{tech:X_compression} and Technique~\ref{tech:D_compression}.

    Also, it is worth noting that our techniques were developed to reduce the computational cost of the classification with an expected accuracy reduction, within acceptable limits. Nevertheless, the classification accuracies on the bark\_woodgrain dataset using our techniques substantially outperforms the accuracies using the original model, as shown in Figure~\ref{fig:binary_dict_size_vs_acc}(a)(bottom). These new higher accuracies were unexpected. Regarding the original models, we noted that the classification accuracies on the training set were 100\% when using dictionaries with at least 50 atoms. These models were probably overfitted to the training set, making them fail to generalize to new data. As our powerize technique introduces a perturbation to the entries of both $\mathbf{D}$ and $\mathbf{w}$, we hypothesize that it reduced the overfitting of $\mathbf{D}$ and $\mathbf{w}$ to the training set and, consequently, increased their generalization power on unseen data \citep{Pfahringer:1995wk}. However, this needs further investigation.

\section{Conclusion}
  \label{Conclusion}
  This paper presented a set of techniques for the reduction of the computations at test time of classifiers that are based on learned transform and soft-threshold. Basically the techniques are: adjust the threshold so the classifier can use signals represented in integer instead of their normalized version in floating-point; reduce the multiplications to simple bit shifts by approximating the entries from both dictionary $\mathbf{D}$ and classifier vector $\mathbf{w}$ to the nearest power of 2; and increase the sparsity of the dictionary $\mathbf{D}$ by applying a hard-thresholding to its entries. We ran simulations using the same datasets used in the original paper that introduces LAST and our results indicate that our techniques substantially reduce the computation load at a small cost of the classification accuracy. Moreover, in one of the datasets tested there was a substantial increase in the accuracy of the classifier. These proposed optimization techniques are valuable in applications where power consumption is critical.

\section*{Acknowledgments}
  \label{Acknowledgments}
  This work was partially supported by a scholarship from the Coordination of Improvement of Higher Education Personnel (Portuguese acronym CAPES). We thank the Dept. of ECE of the UTEP for allowing us access to the NSF-supported cluster (NSF CNS-0709438) used in all the simulations here described and also Mr. N. Gumataotao for his assistance with it. We thank Mr. A. Fawzi for the source code of LAST and all the help with its details. We also thank Dr. G. von Borries for fruitful cooperation and discussions.

  \bibliographystyle{elsarticle-num} 
  \bibliography{my_files/all_papers}

\end{document}